\documentclass[sigconf]{acmart}
\usepackage{graphicx}
\usepackage{xcolor}
\usepackage{multirow}
\usepackage{tabularx}
\usepackage{algorithm}
\usepackage{algorithmic}
\usepackage{subcaption}
\usepackage{amsthm}
\usepackage{enumitem}

\newcolumntype{C}{>{\centering\arraybackslash}X} 
\copyrightyear{2021}
\acmYear{2021}
\setcopyright{acmcopyright}\acmConference[WSDM '21]{Proceedings of the Fourteenth ACM International Conference on Web Search and Data Mining}{March 8--12, 2021}{Virtual Event, Israel}
\acmBooktitle{Proceedings of the Fourteenth ACM International Conference on Web Search and Data Mining (WSDM '21), March 8--12, 2021, Virtual Event, Israel}
\acmPrice{15.00}
\acmDOI{10.1145/3437963.3441752}
\acmISBN{978-1-4503-8297-7/21/03}

\begin{document}
\fancyhead{}
\title{Say No to the Discrimination: Learning Fair Graph Neural Networks with Limited Sensitive Attribute Information}
\author{Enyan Dai, Suhang Wang}
\affiliation{The Pennsylvania State University}
\email{{emd5759, szw494}@psu.edu}

\begin{abstract}
Graph neural networks (GNNs) have shown great power in modeling graph structured data. However, similar to other machine learning models, GNNs may make predictions biased on protected sensitive attributes, e.g., skin color and gender. 
Because machine learning algorithms including GNNs are trained to reflect the distribution of the training data which often contains historical bias towards sensitive attributes. In addition, the discrimination in GNNs can be magnified by graph structures and the message-passing mechanism. As a result, the applications of GNNs in sensitive domains such as crime rate prediction would be largely limited. Though extensive studies of fair classification have been conducted on i.i.d data, methods to address the problem of discrimination on non-i.i.d data are rather limited. Furthermore, 
the practical scenario of sparse annotations in sensitive attributes is rarely considered in existing works. Therefore, we study the novel and important problem of learning fair GNNs with limited sensitive attribute information. FairGNN is proposed to eliminate the bias of GNNs whilst maintaining high node classification accuracy by leveraging graph structures and limited sensitive information. Our theoretical analysis shows that FairGNN can ensure the fairness of GNNs under mild conditions given limited nodes with known sensitive attributes. Extensive experiments on real-world datasets also demonstrate the effectiveness of FairGNN in debiasing and keeping high accuracy.

\end{abstract}


\maketitle

\section{introduction}

Graph neural networks (GNNs)~\cite{bruna2013spectral,kipf2016semi,velivckovic2017graph,hamilton2017inductive} have achieved remarkable performance on various domains such as knowledge graph~\cite{hamaguchi2017knowledge,wang2018cross}, social media mining~\cite{hamilton2017inductive}, nature language processing~\cite{kipf2016semi,yao2019graph}, and recommendation system~\cite{ying2018graph,berg2017graph}. Generally, message-passing process 
is adopted in GNNs~\cite{kipf2016semi,hamilton2017inductive}, where information from neighbors is aggregated for every node in each layer. This process enriches node representations, and preserves both node feature characteristics and topological structures. 

Despite the success in modeling graph data, GNNs trained on graphs may inherit the societal bias in data, which limits the adoption of GNNs in many real-world applications. \textit{First}, extensive studies~\cite{dwork2012fairness,beutel2017data,creager2019flexibly} have revealed that historical data may include patterns of previous discrimination and societal bias. Machine learning models trained on such data can inherit the bias on sensitive attributes such as ages, genders, skin color, and regions~\cite{dwork2012fairness,beutel2017data}, which implies that GNNs could also exhibit the bias. \textit{Second}, the topology of graphs and the message-passing of GNNs could magnify the bias. Generally, in graphs such as social networks, nodes of similar sensitive attributes are more likely to connect to each other than nodes of different sensitive attributes~\cite{dong2016young,rahman2019fairwalk}. For example, young people tend to build friendship with people of similar age on the social network \cite{dong2016young}. 
This makes the aggregation of neighbors' features in GNN have similar representations for nodes of similar sensitive information while different representations for nodes of different sensitive features, leading to severe bias in decision making, i.e., the predictions are highly correlated with the sensitive attributes of the nodes. Our preliminary experiments in Sec.~\ref{sec:dis} indicate that GNNs have a larger bias due to the adoption of graph structure than models which only use node attributes, which verifies our intuition. The bias would largely limit the wide adoption of GNNs in domains such as ranking 
of job applicants~\cite{mehrabi2019survey} and crime rate prediction~\cite{suresh2019framework}. Thus, it is important to investigate fair GNNs.

However, developing fair GNNs is a non-trivial task. \textit{First}, to achieve fairness, we need to obtain abundant nodes with known sensitive attributes so that we can either revise the data or regularize the model; whereas people are unwilling to share their sensitive information in the real-world, and resulting in inadequate nodes with sensitive attributes known for fair model learning.
For example, only 14\% teen users public their complete profiles on Facebook~\cite{madden2013teens}. 
The lacking of sensitive information challenges many existing work on fair models~\cite{beutel2017data,locatello2019fairness,louizos2015variational,creager2019flexibly}. \textit{Second}, though extensive efforts have been made to establish fair models by revising features \cite{zhang2017achieving,kamiran2009classifying,kamiran2012data}, disentanglement \cite{louizos2015variational,creager2019flexibly}, adversarial debiasing \cite{edwards2015censoring,beutel2017data} and fairness constraints \cite{zafar2015fairness,zafar2017fairness},
they are overwhelmingly dedicated to independently and identically distributed (i.i.d) data, which cannot be directly applied on graph data for the absence of simultaneous consideration of the bias from node attributes and graph structures.
Recently, \cite{rahman2019fairwalk,bose2019compositional} aim to learn fair node representations from  graphs. These methods merely deal with plain graphs without any node attributes, and focus on fair node representations instead of fair node classifications.

Therefore, in this paper, we study a novel problem of learning fair graph neural networks with limited sensitive information. In essence, we need to solve two challenges: (i) how to overcome the shortage of sensitive attributes for eliminating discrimination; and (ii) how to ensure the fairness of the GNN classifier. In an attempt to address these challenges, we propose a novel framework named as \textbf{FairGNN} for fair node classification. 
A GNN sensitive attribute estimator is adopted in FairGNN to predict plenty of sensitive attributes with noise for fair classification. Inspired by existing works of fair classification on i.i.d data with adversarial learning \cite{edwards2015censoring,beutel2017data,zhang2018mitigating,madras2018learning}, we deploy an adversary to ensure the GNN classifier make predictions independent with the estimated sensitive attributes.
To further stabilize the training process and performance in fairness, we introduce a fairness constraint to make the predictions invariant with the estimated sensitive attributes.
Our main contributions are: 
\begin{itemize}
    \item We study a novel problem of fair graph neutral networks learning with limited sensitive information;
    \item A new framework, FairGNN, is proposed to settle the shortage of sensitive attributes for adversarial debiasing and fairness constraint by estimating users' sensitive attributes;
    \item We conduct theoretical analysis showing fairness achieves at the global minimum even with estimated sensitive attributes;
    \item Extensive experiments on different datasets demonstrate the effectiveness of our methods in eliminating discrimination while keeping high accuracy of GNNs.
\end{itemize}

The rest of the paper is organized as follows. In Sec.~\ref{sec:related_work}, we review related work. In Sec.~\ref{sec:preliminary_analysis}, we  conduct preliminary analysis to understand the bias issue of GNNs. In Sec.~\ref{sec:methdology}, we give the details of FairGNN. In Sec.~\ref{sec:experiments}, we conduct experiments to show the effectiveness of FairGNN. In Sec.~\ref{sec:conclusion}, we conclude with future work.

\section{related work} \label{sec:related_work}
In this section, we will review related work including graph neural networks and fairness in machine learning.

\subsection{Graph Neural Networks}
Graph neural networks (GNNs), which generalize neural networks for graph structured data, have shown great success for various applications~\cite{hamaguchi2017knowledge,yao2019graph,ying2018graph,tang2020investigating,zhao2020semi,sun2019node,tang2020transferring}. 
Generally, GNNs can be categorized into two categories, i.e., spectral-based~\cite{bruna2013spectral,henaff2015deep,defferrard2016convolutional,kipf2016semi,levie2018cayleynets} and spatial-based~\cite{velivckovic2017graph,hamilton2017inductive,chiang2019cluster,ying2018graph}.
Spectral-based GNNs define graph convolution based on spectral graph theory, which is first explored by \citeauthor{bruna2013spectral} \cite{bruna2013spectral}. Since then, more spectral-based methods are developed for further improvements and extensions \cite{henaff2015deep,defferrard2016convolutional,kipf2016semi,levie2018cayleynets}. Graph Convolutional Network (GCN)~\cite{kipf2016semi} is a particularly popular method which simplifies the convolutional operation on the graph. Spatial-based graph convolution directly updates the node representation by aggregating its neighborhoods' representations \cite{niepert2016learning,gilmer2017neural,hamilton2017inductive,ying2018graph}. 
\citeauthor{velivckovic2017graph} \cite{velivckovic2017graph} introduce the self-attention into the aggregation of spatial graph convolution by assigning higher weights to the more important nodes in graph attention network (GAT). Various spatial methods are proposed to solve the scalability issue of GCN \cite{hamilton2017inductive,chiang2019cluster}.  For example, a neighbor sampling method to train GNN with nodes in mini-batch instead of the whole graph is developed in GraphSAGE~\cite{hamilton2017inductive}.
Moreover, spatial-based methods have already been successfully deployed to deal with extremely large industrial datasets \cite{ying2018graph}.

The essential idea of GNNs is to propagate the information of nodes through the graph to get better representations. However, people tend to build relationships with those sharing the same sensitive attributes. Then, representations in GNNs are nearly propagated within the subgroup, which highly increases the risk of discrimination towards sensitive attributes. Despite the risk of discrimination in GNNs, there is no existing work to address this problem. Thus, we study the novel problem of learning fair GNNs to eliminate the potential discrimination.

\subsection{Fairness in Machine Learning}
Many works have been conducted to deal with the bias in the training data to achieve fairness in machine learning~\cite{zhang2017achieving,kamiran2009classifying,kamiran2012data,beutel2017data,locatello2019fairness,dwork2012fairness,hardt2016equality}. Based on which stage of the machine learning training process is revised, algorithms could be split into three categories: the pre-processing approaches, the in-processing approaches, and the post-processing approaches. 
The pre-processing approaches are applied before training machine learning models. They could reduce the bias by modifying the training data through correcting labels \cite{zhang2017achieving,kamiran2009classifying}, revising attributes of data \cite{kamiran2012data,feldman2015certifying}, generating non-discriminatory labeled data \cite{xu2018fairgan,xu2019fairgan+,sattigeri2019fairness}, and obtaining fair data representations \cite{beutel2017data,locatello2019fairness,edwards2015censoring,zemel2013learning,louizos2015variational,creager2019flexibly}. The in-processing approaches are designed to revise the training of the state-of-the-art models. 
Typically the machine learning models are trained with
additional regularization terms or a new objective function. \cite{dwork2012fairness,zafar2015fairness,kamishima2011fairness,zhang2018mitigating}. Finally, the post-processing approaches directly change the predictive labels to ensure fairness \cite{hardt2016equality,pleiss2017fairness}. 
Recently, several works explore the learning of fair graph embeddings for recommendation \cite{rahman2019fairwalk,bose2019compositional}. Fairwalk \cite{rahman2019fairwalk} modifies the random walk procedure of node2vec \cite{grover2016node2vec} to obtain a more diverse network neighborhood representations. The sensitive attributes of all the nodes are required in the sampling procedure of FairWalk. \citeauthor{bose2019compositional} \cite{bose2019compositional} propose to add discriminators 
to eliminate the sensitive information in the graph embeddings. Similar to Fairwalk, the training process of the discriminators is in need of the sensitive attributes of all the nodes.

Our work is inherently different from existing works: (i) we focus on learning fair GNNs for node classification instead of fair graph embeddings; (ii)  we address the problem that only a limited number of nodes are provided with sensitive attributes in practice.

\section{Preliminaries Analysis} \label{sec:preliminary_analysis}
In this section, we first conduct preliminary analysis on real-world datasets to show that GNNs could exhibit more serve bias due to the graph structure and the message-passing. Sequentially, We formally give the problem definition of fair node classification.  

\subsection{Notations} \label{sec:notations}
We use $\mathcal{G}=(\mathcal{V},\mathcal{E}, \mathbf{X})$ to denote an attributed graph, where $\mathcal{V}=\{v_1,...,v_N\}$ is the set of $N$ nodes, $\mathcal{E} \subseteq \mathcal{V} \times \mathcal{V}$ is the set of edges, and $\mathbf{X}=\{\mathbf{x}_1,...,\mathbf{x}_N\}$ is the set of node features. $\mathbf{A} \in \mathbb{R}^{N \times N}$ is the adjacency matrix of the graph $\mathcal{G}$, where $\mathbf{A}_{ij}=1$ if nodes ${v}_i$ and ${v}_j$ are connected; otherwise, $\mathbf{A}_{ij}=0$. In the semi-supervised setting, part of nodes $v \in \mathcal{V}_L$ are provided with labels $y_v \in \mathcal{Y}$, 
where $\mathcal{V}_L \subseteq \mathcal{V}$ denotes nodes with labels, and $\mathcal{Y}$ is the set of labels. Sensitive attributes of training nodes are required to achieve fairness of machine learning algorithms. In our setting, only a small set of nodes $\mathcal{V}_S \subset \mathcal{V}$ are provided with the sensitive attribute $s \in \{0,1\} $. The set of provided sensitive attributes is denoted by $\mathcal{S}$.

\subsection{Datasets} \label{sec:datasets}
\begin{table}[t]
    \small
    \centering
    \caption{The statistics of datasets.}
    \vskip -1.5em
    \begin{tabularx}{0.88\linewidth}{p{0.35\linewidth}XXX}
    \toprule
    Dataset   & Pokec-z & Pokec-n  &  NBA  \\
    \midrule
    \# of nodes  & 67,797 & 66,569 & 403\\
    \# of node attributes  & 59 & 59 & 39\\
    \# of edges  & 882,765 & 729,129 & 16,570 \\
    Size of $\mathcal{V}_L$  & 500 & 500 & 100\\
    Size of $\mathcal{V}_S$ & 200 & 200 & 50 \\
    Group ratio & 1.84 & 2.46  & 2.77\\
    \# of inter-group edges  & 39,804 & 31,515 & 4,401\\
    \# of intra-group edges  & 842,961 & 697,614 & 12,169\\
    \bottomrule
    \end{tabularx}
    \label{tab:stat}
    \vskip -1.5em
\end{table}
For the purpose of this study, we collect and sample datasets from Pokec and NBA. The details are described as below.

\textbf{Pokec \cite{takac2012data}}: It is the most popular social network in Slovakia, which is very similar to Facebook and Twitter. This dataset contains  anonymized data of the whole social network in 2012. User profiles of Pokec contain gender, age, hobbies, interest, education, working field and etc. The original Pokec dataset contains millions of users. Based on the provinces that users belong to, we sampled two datasets named as: \textbf{Pokec-z} and \textbf{Pokec-n}. Both Pokec-z and Pokec-n consist of users belonging to two major 
regions of the corresponding provinces. We treat the region as the sensitive attribute. The classification task is to predict the working field of the users. 

\textbf{NBA}: This is extended from a Kaggle dataset~\footnote{https://www.kaggle.com/noahgift/social-power-nba} containing around 400 NBA basketball players. 
The performance statistics of players in the 2016-2017 season and other various information e.g., nationality, age, and salary are provided. To obtain the graph that links the NBA players together, we collect the relationships of the NBA basketball players on Twitter with its official crawling API~\footnote{https://developer.twitter.com/en}. We binarize the nationality to 
two categories, i.e., U.S. players and oversea players, which is used as sensitive attribute. 
The classification task is to predict whether the salary of the player is over median.

For all the datasets, we eliminate nodes without any links with others. We randomly sample labels and sensitive attributes separately to get $\mathcal{V}_L$ and $\mathcal{V}_S$.
We randomly sample 25\% and 50\%  of nodes containing both sensitive attributes and labels in Pokec-z, Pokec-n and NBA as validation sets and test sets. Note that the validation sets and test sets have no overlap with $\mathcal{V}_L$ and $\mathcal{V}_S$. 
The key statistics of the datasets are given in Table \ref{tab:stat}. Apart from the basic statistics, we also report the ratio of the majority and minority group and the number of edges linking the same group and different groups. It is evident from the table that: (i) skew exists in sensitive attributes; (ii) most of relationships are between users who share the same sensitive attribute. 

\subsection{Preliminaries of Graph Neural Networks} 

Graph neural networks (GNNs) utilize the node attributes and edges to learn a representation $\mathbf{h}_v$ of the node $v \in \mathcal{V}$. The goal of learning representation in node classification is to predict the node $v$'s label as $y_v = f(\mathbf{h}_v)$ . Current GNNs are neighborhood aggregation approaches, which will update the representations of the nodes with the representations of the neighborhood nodes. The representations after $k$ layers' 
aggregation would capture the 
structural information of the $k$-hop network neighborhoods. The updating process of the $k$-th layer in GNN could be formulated as:
\begin{equation}
\begin{aligned}
    \mathbf{a}^{(k)}_v & = \text{AGGREGATE}^{(k-1)}(\{\mathbf{h}^{(k-1)}_u: u \in \mathcal{N}(v)\}),
    \label{eq:GNN_a} \\
    \mathbf{h}^{(k)}_{v} & =\text{COMBINE}^{(k)}(\mathbf{h}^{(k-1)}_v, \mathbf{a}^{(k)}),
\end{aligned}
\end{equation}
where $\mathbf{h}^{(k)}_v$ is the representation vector of the node $v \in \mathcal{V}$ at $k$-th layer and $\mathcal{N}(v)$ is  a set of neighborhoods of $v$.

\subsection{Fairness Evaluation Metrics}
In this subsection, we will present two definitions of fairness for the binary label $y \in \{0,1\}$ and the sensitive attribute $s \in \{0,1\}$. $\hat{y} \in  \{0, 1\}$ denotes the prediction of the classifier $\eta$: $\mathbf{x} \rightarrow y$.
\begin{definition}
(Statistical Parity \cite{dwork2012fairness}). Statistical parity requires the predictions to be independent with the sensitive attribute $s$, i.e., $\hat{y} \bot s$. It could be formally written as:
\begin{equation}
    P(\hat{y}|s=0)=P(\hat{y}|s=1). 
\end{equation}
\end{definition}
\begin{definition} (Equal Opportunity \cite{hardt2016equality}).
Equal opportunity requires the probability of an instance in a positive class being assigned to a positive outcome should be equal for both subgroup members. The property of equal opportunity is defined as:
 \begin{equation}
     P(\hat{y}=1|y=1,s=0) = P(\hat{y}=1|y=1,s=1). 
 \end{equation}
The equal opportunity expects the classifier to give equal true positive rates across the subgroups.
According to~\cite{louizos2015variational,beutel2017data}, we apply the following metrics to quantitatively evaluate statistical parity and equal opportunity:
 \begin{equation}
      \Delta_{SP}  = |P(\hat{y}=1|s=0)-P(\hat{y}=1|s=1)|,
 \end{equation}
 \begin{equation}
     \Delta_{EO} = |P(\hat{y}=1|y=1,s=0) - P(\hat{y}=1|y=1,s=1)|,
 \end{equation}
 where the probabilities are evaluated on the test set.
\end{definition}
\label{sec:fairness}
\subsection{Discrimination in Graph Neural Networks}

\begin{table}[t]
    \small
    \centering
    \caption{Results of models w/ and w/o utilizing graph.}
    \vskip -1.5em
    \begin{tabularx}{\linewidth}{|p{0.1\linewidth}|p{0.12\linewidth}|X|X|X|X|}
    \hline
    Dataset & Metrics & MLP & MLP-e & GCN & GAT \\
    \hline
    \hline
    \multirow{4}{*}{Pokec-z}
        & ACC (\%) & 65.3 $\pm 0.5$ & 68.6 $\pm 0.3$ & 70.2  $\pm 0.1$ & 70.4 $\pm 0.1$ \\
        & AUC (\%)& 71.3 $\pm 0.3$ & 74.8 $\pm 0.3$ & 77.2  $\pm 0.1$ & 76.7 $\pm 0.1$\\
        & $\Delta_{SP}$ (\%)& 3.8 $\pm 1.3$ & 6.9 $\pm 1.0$ & 9.9  $\pm 1.1$ & 9.1 $\pm 0.9$  \\
        & $\Delta_{EO}$ (\%)& 2.2 $\pm 0.7$ & 4.0 $\pm 1.5$ & 9.1  $\pm 0.6$ & 8.4  $\pm 0.6$ \\
        \hline
        \hline
    \multirow{4}{*}{Pokec-n}   
        & ACC (\%) & 63.1 $\pm 0.4$ & 66.3 $\pm 0.6$ & 70.5 $\pm 0.2$ & 70.3 $\pm 0.1$\\
        & AUC (\%)& 68.2 $\pm 0.3$ & 72.4 $\pm 0.6$  & 75.1 $\pm 0.2$ & 75.1 $\pm 0.2$\\
        & $\Delta_{SP}$ (\%)& 3.3 $\pm 0.6$ & 8.7 $\pm 1.0$ & 9.6 $\pm 0.9$ & 9.4 $\pm 0.7$ \\
        & $\Delta_{EO}$ (\%)& 7.1 $\pm 0.9$ & 9.9 $\pm 0.6$ & 12.8 $\pm 1.3$ & 12.0 $\pm 1.5$ \\
        \hline
        \hline
    \multirow{4}{*}{NBA}     & ACC (\%)& 63.6 $\pm 0.9$ & 66.1 $\pm 1.1$ & 71.2 $\pm 0.5$ & 71.9 $\pm 1.1$\\
        & AUC (\%)& 73.5 $\pm 0.3$ & 74.4 $\pm 1.2$ & 78.3 $\pm 0.3$ & 78.2 $\pm 0.6$\\
        & $\Delta_{SP}$ (\%)& 6.0$\pm 1.5$ & 10.9 $\pm 1.9$ & 7.9 $\pm 1.3$ & 10.2 $\pm 2.5$\\
        & $\Delta_{EO}$ (\%)& 6.1 $\pm 1.8$ & 8.8 $\pm 3.0 $ & 17.8 $\pm 2.6$ & 15.9 $\pm 4.0$ \\
    \hline
    \end{tabularx}
    
    \label{tab:pre}
    \vspace{-3mm}
\end{table}
Various machine learning algorithms such as logistic regression \cite{zafar2015fairness}, SVM \cite{zafar2015fairness}, and MLP \cite{edwards2015censoring} have been reported to have discrimination. The features of the instances may contain proxy variables of the sensitive attribute. It could result in biased predictions. For GNNs, edges in graph can bring linking bias, i.e., the misrepresentation due to the connections of users \cite{mehrabi2019survey}.
It has been proven that the embeddings of nodes within the connected component will be closer after one aggregation in GCN \cite{li2018deeper,wang2020unifying}. Since most of edges are intra-group as Table \ref{tab:stat} shows, embeddings of nodes sharing the same sensitive attribute will be closer after $k$-layer information aggregation. As a result, representations of the nodes may exhibit bias. 
Intuitively, similar discrimination also exists in other GNNs that aggregate information of neighborhoods.

To empirically demonstrate the existence of discrimination in GNNs, we make comparisons between the following models:
\begin{itemize}
    \item \textbf{MLP}: A multi-layer perception model trained on $\mathcal{V}_L$.
    \item \textbf{MLP-e}: A MLP model utilizes graph structure by adding embeddings learned by deepwalk to the features.
    \item \textbf{GCN} \cite{kipf2016semi}: A state-of-the-art spectral graph neural network.
    \item \textbf{GAT} \cite{velivckovic2017graph}: A spatial graph neural network which utilizes attention to assign higher weights to more important edges.
\end{itemize}
For each model, we run the experiment 5 times. The classification results and discrimination scores on the test set are reported in Table \ref{tab:pre}. 
From the table, we observe that (i) both performance of GCN and GAT are much better than MLP, which is as expected because GCN and GAT adopt both node attributes and the graph structure for classification; (ii) Compared with MLP, models utilizing graph structure, i.e., GCN and GAT, perform significantly worse in terms of fairness, which verifies that \textit{bias exists in GNNs and the graph structure could further aggravate the discrimination.}
\label{sec:dis}
\subsection{Problem definition}
Our preliminary analysis verifies that GNNs have severe bias issue. Thus, it is important to develop fair GNNs. Following existing work of fair models \cite{louizos2015variational,feldman2015certifying,beutel2017data,xu2018fairgan}, we focus on the binary class and binary sensitive attribute setting, i.e., both $y$ and $s$ can either be 0 or 1. We leave the extension to multi-class and multi-sensitive attribute setting as a future work.
With the notations given in Section \ref{sec:notations}, 
the fair GNN problem is formally defined as:

\vspace*{-0.3em}

\newtheorem{problem}{Problem}
\begin{problem}
Given a graph $\mathcal{G}=(\mathcal{V},\mathcal{E}, \mathbf{X})$, small labeled node set $\mathcal{V}_L \in \mathcal{V}$ with the corresponding labels in $\mathcal{Y}$, and a small set of nodes $\mathcal{V}_S \in \mathcal{V}$ with corresponding sensitive attributes in $\mathcal{S}$, learn a fair GNN for fair node classification, i.e.,
\begin{equation}
    f(\mathcal{G}, \mathcal{Y}, \mathcal{S}) \rightarrow \hat{\mathcal{Y}}
\end{equation}
where $f$ is the function we aim to learn and $\hat{\mathcal{Y}}$ is the set of predicted labels for unlabeled nodes.  $\hat{\mathcal{Y}}$ should maintain high accuracy whilst satisfying the 
fairness criteria such as statistical parity.
\end{problem}

\section{methodology} \label{sec:methdology}

In this section, we give the details of FairGNN. An illustration of the proposed framework is shown in Figure \ref{fig:framework}, which is composed of a GNN classifier $f_{\mathcal{G}}$, a GCN based sensitive attribute estimator $f_E$ and an adversary $f_A$. The classifier $f_{\mathcal{G}}$ takes $\mathcal{G}$ as input for node classification.  The sensitive attribute estimator $f_E$ is to predict the sensitive attributes for nodes whose sensitive attributes are unknown, which paves us a way to adopt adversarial learning to learn fair node representations and to regularize the predictions of $f_{\mathcal{G}}$. Specifically, the adversary $f_A$ aims to predict the known or estimated sensitive attributes by $f_E$ from the node representation learned by $f_\mathcal{G}$; while $f_\mathcal{G}$ aims to learn fair node representations that can fool the adversary $f_A$ to make wrong predictions. We theoretically prove that under mild conditions, such minmax game can guarantee that learned representations are fair. 
In addition to make the representations fair, we directly add a regularizer on the predictions of $f_{\mathcal{G}}$ to guarantee that $f_{\mathcal{G}}$ gives fair predictions. Next, we introduce each component in detail along with theoretical proof.

\begin{figure}
    \centering
    \includegraphics[width=0.90\linewidth]{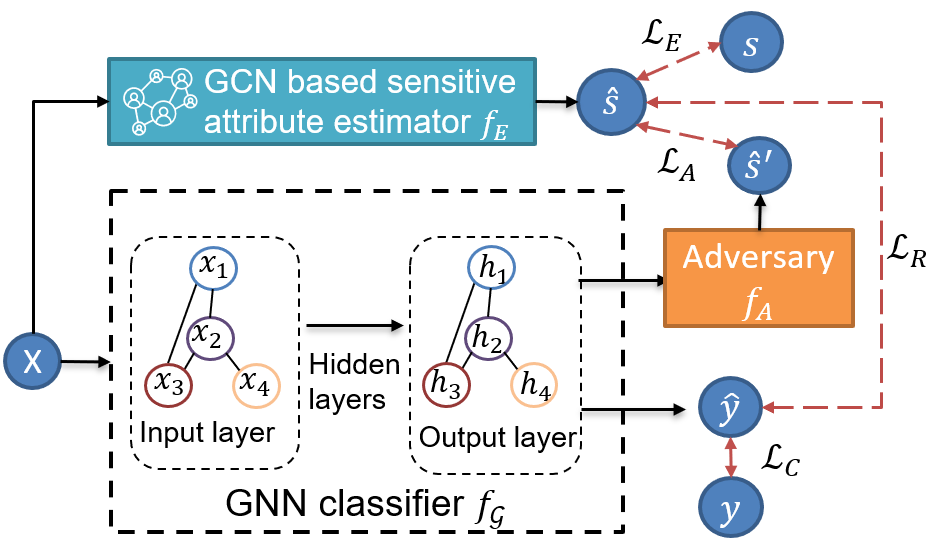}
    \vskip -1.5em
    \caption{The overall framework of FairGNN.}
    \vskip -1.5em
    \label{fig:framework}
\end{figure}


\subsection{The GNN Classifier $f_{\mathcal{G}}$}
The GNN classifier $f_{\mathcal{G}}$ takes $\mathcal{G}$ as input and predicts node labels. The proposed framework FairGNN is flexible. Any GNNs that follow the structure of Eq.(\ref{eq:GNN_a}) can be used such as GCN~\cite{kipf2016semi} and GAT~\cite{velivckovic2017graph}. Let $f_{\mathcal{G}}^{(k)}$ denote the operation of aggregating and combining the information of node $v$ and its $k$-hop neighborhoods through $k$ layers' iterations in GNN classifier $f_{\mathcal{G}}$. For a GNN with $K$ layers, the representation of node $v$ of the final layer could be written as:
\begin{equation}
    \mathbf{h}_v = f_{\mathcal{G}}^{(K)}(\mathbf{x}_v,\mathcal{N}^{(K)}_v), 
\end{equation}
where $\mathcal{N}^{(K)}_v$ represents the $K$-hop neighborhoods of $v$. To get the $\hat{y}_v$, i.e., the prediction of node $v$, a linear classification layer is applied to $\mathbf{h}_v$ as:
\begin{equation}
    \hat{y}_v = \sigma(\mathbf{h}_v \cdot \mathbf{w}),
    \label{eq:y_v}
\end{equation}
where $\mathbf{w} \in \mathbb{R}^{d}$ is the weights of the linear classification layer and $\sigma$ is the sigmoid function. The loss function for training $f_{\mathcal{G}}$ is
\begin{equation} \label{eq:obj_GNN_classifier}
    \min_{\theta_{\mathcal{G}}}\mathcal{L}_C = -\frac{1}{|\mathcal{V}_L|}\sum_{v\in \mathcal{V}_L} [y_v \log{\hat{y}_v} + (1-y_v) \log{(1 - \hat{y}_v})],
\end{equation}
where $|\mathcal{V}_L|$ denotes the size of $\mathcal{V}_L$, $\theta_{f_{\mathcal{G}}}$ represents the parameters of $f_{\mathcal{G}}$ and $y_v$ is the groundtruth label of node $v$.

\subsection{Adversarial Debiasing with Estimator $f_E$}
The GNN classifier $f_\mathcal{G}$ can make biased predictions because the learned representations of $f_\mathcal{G}$  exhibit bias due to the node features, graph structure and aggregation mechanism of GNN. One way to make $f_\mathcal{G}$ fair is to eliminate the bias of the final layer representations $\mathbf{h}_{v}$. Recently, adversarial debiasing has been proven to be effective in alleviating the bias of representations  \cite{beutel2017data,edwards2015censoring,liao2019learning,madras2018learning}.
In the general process of adversarial debiasing, an adversary is used to predict sensitive attributes from the representations of the classifier; while the classifier is trained to learn representations to make the adversary unable to predict the sensitive attributes while keep high accuracy in the classification task. Such process requires \textit{abundant data samples with known sensitive attributes} so that we can judge if the adversary can make accurate predictions or not.

However, in practice people are reluctant to share their sensitive attributes, which leads to small size $\mathcal{V}_S$. Lacking of data with labeled sensitive attributes would result in poor improvement in fairness even with adversarial debiasing. Though we have limited nodes with sensitive attributes, i.e., small $\mathcal{V}_S$, generally, nodes with similar sensitive attributes are more likely connected to each other, which makes it possible to accurately predict the sensitive attributes for nodes in $\mathcal{V} - \mathcal{V}_S$ using the graph $\mathcal{G}$ and $\mathcal{V}_S$. Thus, we deploy a graph convolutional network $f_E: \mathcal{G} \rightarrow \mathcal{S}$ to estimate the sensitive attribute of node whose sensitive attribute is unavailable. The large amount of estimated sensitive attributes would greatly 
benefit the adversarial debiasing. Note that it is important to use two separate GNNs for node label prediction and sensitive attribute prediction because we aim to learn fair representations $\mathbf{h}_v$ for $f_{\mathcal{G}}$, i.e., $\mathbf{h}_v$ does not contain the sensitive information. The objective function of training $f_E$ is
\begin{equation}
    \min_{\theta_E} \mathcal{L}_E = -\frac{1}{|\mathcal{V}_S|} \sum_{v\in \mathcal{V}_S} [s_v \log{\hat{s}_v} + (1-s_v) \log{(1 - \hat{s}_v})],
    \label{eq:L_E}
\end{equation}
where $\hat{s}_v$ is the predicted sensitive attribute of node $v \in \mathcal{V}_S$ by  $f_E$ and $\theta_E$ is the set of parameters of $f_E$. 

With $f_E$, we could get the estimation of the sensitive attributes $\hat{\mathcal{S}}_u$ of the  nodes $u \in (\mathcal{V} - \mathcal{V}_S)$. 
We use $\hat{\mathcal{S}}$ to denote the set of sensitive attributes by combining $\mathcal{S}$ and $\hat{\mathcal{S}}_u$, i.e., $\hat{\mathcal{S}} = \mathcal{S} \cup \hat{\mathcal{S}}_u$.
During the training process, for each node $v \in \mathcal{V}$, the adversary 
$f_A$ tries to predict $v$'s sensitive attribute $\hat{s}_v$ given the representation $\mathbf{h}_v$ as $f_A(\mathbf{h}_v)$; 
while  $f_{\mathcal{G}}$ aims to learn node representation $\mathbf{h}_v$ that makes the adversary $f_A$ unable to distinguish which sensitive group the node $v$ belong to. This min max game can be written as
\begin{equation}
\begin{aligned}
\min_{\theta_{\mathcal{G}}} \max_{\theta_A} \mathcal{L}_A & = \mathbb{E}_{\mathbf{h} \sim p(\mathbf{h}|\hat{s}=1)} [\log(f_A(\mathbf{h}))] \\
 & + \mathbb{E}_{\mathbf{h} \sim p(\mathbf{h}|\hat{s}=0)} [ \log(1-f_A(\mathbf{h}))],
\end{aligned}
\label{eq:A_S}
\end{equation}
where $\mathbf{h} \sim p(\mathbf{h}|\hat{s}=1)$ means sampling a node with sensitive attribute as 1 from $\mathcal{G}$. $\theta_A$ is the parameters of $f_A$.  

\textbf{\textbf{Theoretical Analysis.}}
Since the size of $\mathcal{V}_S$ is small, the estimation of sensitive attributes will introduce nonnegligible noise. The noise of the sensitive attributes may influence the adversarial debiasing. Thus, we conduct theoretical analysis to show that sensitive attributes containing noise could help to achieve statistical parity under mild conditions. Next, we give the details of the proof. 
\begin{proposition}
The global minimum of Eq.(\ref{eq:A_S}) is achieved if and only if $p(\mathbf{h}|\hat{s}=1)=p(\mathbf{h}|\hat{s}=0)$, where $\hat{s} \in \hat{\mathcal{S}}$ and $\mathbf{h}$ is final layer representation learned by the $K$-layer GNN classifier $f_{\mathcal{G}}$.
\label{prop:A_s}
\end{proposition}
\begin{proof}
According to Proposition 1. in \cite{goodfellow2014generative}, the optimal adversary is 
$f_A^*(\mathbf{h})=\frac{p(\mathbf{h}|\hat{s}=1)}{p(\mathbf{h}|\hat{s}=1) + p(\mathbf{h}|\hat{s}=0)}$. Then the min max game in Eq.(\ref{eq:A_S}) could be reformulated as minimizing this function:
\begin{equation}
\begin{aligned}
 C^s & = \mathbb{E}_{\mathbf{h} \sim p(\mathbf{h}|\hat{s}=1)} \big[\log{\frac{p(\mathbf{h}|\hat{s}=1)}{p(\mathbf{h}|\hat{s}=1) + p(\mathbf{h}|\hat{s}=0)}}\big] \\
 & + \mathbb{E}_{\mathbf{h} \sim p(\mathbf{h}|\hat{s}=0)} \big[ \log{\frac{p(\mathbf{h}|\hat{s}=0)}{p(\mathbf{h}|\hat{s}=1) + p(\mathbf{h}|\hat{s}=0)}} \big] \\
 & = -\log(4) + 2\cdot JSD(p(\mathbf{h}|\hat{s}=1)||p(\mathbf{h}|\hat{s}=0).
\end{aligned}
\end{equation}
The Jensen-Shannon divergence between two distributions is non-negative, and become zero if the two distributions are equal. Thus, only if $p(\mathbf{h}|\hat{s}=1)=p(\mathbf{h}|\hat{s}=0)$, the objective function $C^s$ will reach the minimum, which completes our proof.
\end{proof}
\begin{theorem}
Let $\hat{y}$ denote the prediction of $f_{\mathcal{G}}$.
Suppose:
\begin{enumerate}
    \item The estimated sensitive attribute $\hat{s} $ and $\mathbf{h}$ are independent conditioned on true sensitive attribute  $s$, i.e., $p(\hat{s},\mathbf{h}|s) = p(\hat{s}|s)p(\mathbf{h}|s)$;
    \item $p(s=1|\hat{s}=1) \neq p(s=1|\hat{s}=0)$.
\end{enumerate}
If Eq.(\ref{eq:A_S}) reaches the global minimum, the GNN classifier $f_{\mathcal{G}}$ will achieve statistical parity, i.e., $p(\hat{y}|s=0)=p(\hat{y}|s=1)$.
\label{theorem:adv}
\end{theorem}
\begin{proof}
Under the assumption that $p(\hat{s},\mathbf{h}|s) = p(\hat{s}|s)p(\mathbf{h}|s)$, we could obtain $p(\mathbf{h}|s,\hat{s}) = p(\mathbf{h}|s)$.
From Proposition \ref{prop:A_s}, we have  $p(\mathbf{h}|\hat{s}=1) = p(\mathbf{h}|\hat{s}=0)$ when the algorithm converges, which is equivalent to $\sum_s p(\mathbf{h},s|\hat{s}=1) = \sum_s p(\mathbf{h},s|\hat{s}=0)$. Together with $p(\mathbf{h}|s,\hat{s}) = p(\mathbf{h}|s)$, we arrive at 
\begin{equation}
    \begin{aligned}
    \sum_s p(\mathbf{h}|s)p(s|\hat{s}=1) = \sum_s p(\mathbf{h}|s)p(s|\hat{s}=0)
    \end{aligned}
    \label{eq:13}
\end{equation}
Reordering the terms in Eq.(\ref{eq:13}), we can get
\begin{equation}
    \begin{aligned}
    \frac{p(\mathbf{h}|s=1)}{p(\mathbf{h}|s=0)} & = \frac{p(s=0|\hat{s}=1)-p(s=0|\hat{s}=0)}{p(s=1|\hat{s}=0)-p(s=1|\hat{s}=1)} \\
    & = \frac{(1-p(s=1|\hat{s}=1))-(1-p(s=1|\hat{s}=0))}{p(s=1|\hat{s}=0)-p(s=1|\hat{s}=1)} \\
    & = 1
    \end{aligned}
    \label{eq:ratio}
\end{equation}
Eq.(\ref{eq:ratio}) shows that at the global minimum $p(\mathbf{h}|s=1) = p(\mathbf{h}|s=1)$ under the assumption  $p(s=1|\hat{s}=1) \neq p(s=1|\hat{s}=0)$.
Since $\hat{y} = \sigma(\mathbf{h} \cdot \mathbf{w})$, we could get $p(\hat{y}|s=1) = p(\hat{y}|s=0)$. Thus, the statistical parity is achieved when Eq.(\ref{eq:A_S}) converges. 
\end{proof}

In our proof, two assumptions are made. For the first assumption, since we use $f_E$ to predict the sensitive attributes $\hat{s}$ and $f_{\mathcal{G}}$ to get the latent representation $\mathbf{h}$, and $f_E$ and $f_{\mathcal{G}}$ doesn't share any parameters, it is generally true that $\hat{s}$ is independent with the representation $\mathbf{h}$, i.e., $p(\hat{s},\mathbf{h}|s) = p(\hat{s}|s)p(\mathbf{h}|s)$.
As for the second assumption, it will be satisfied when we have a reasonable estimator $f_E$, i.e., $f_E$ doesn't give random predictions.

\subsection{Covariance Constraint}
The instability of the training process of adversarial learning is well known \cite{arjovsky2017towards}. In adversarial debiasing, failure to coverage may result in a classifier with discrimination. 
To alleviate this issue, we add a covariance constraint~\cite{zafar2015fairness,zafar2017fairness} on the output of $f_{\mathcal{G}}$ to help the model achieve fairness. 
The covariance constraint has been explored in~\cite{zafar2015fairness,zafar2017fairness} by minimizing the absolute covariance between users' sensitive attributes and the signed distance from the users' features to the decision boundary for fair linear classifiers. 
In our problem, only a small portion of users' sensitive attributes are known and the decision boundary of GNN is hard to obtain. Thus, we propose to minimize the absolute covariance between the noisy sensitive attribute $\hat{s} \in \hat{\mathcal{S}}$ and prediction $\hat{y}$ as
\begin{equation}
    \mathcal{L}_R = |\text{Cov}(\hat{s},\hat{y})| = |\mathbb{E}[(\hat{s}-\mathbb{E}(\hat{s}))(\hat{y}-\mathbb{E}(\hat{y}))]|,
\end{equation}
where $|\cdot|$ indicates the absolute value. 

\textbf{Theoretical Analysis.}
Since $\mathcal{L}_R$ is the absolute value of covariance between $\hat{y}$ and $\hat{s}$, $\mathcal{L}_R=0$, i.e., the global minimum of $\mathcal{L}_R$, is the prerequisite that $\hat{y}$ and $\hat{s}$ are independent.
Thus, we will show that $\mathcal{L}_R=0$ is the prerequisite of the statistical parity under mild assumption with the following theorem.
\begin{theorem}
Suppose that $p(\hat{s},\mathbf{h}|s) = p(\hat{s}|s)p(\mathbf{h}|s)$, when $f_{\mathcal{G}}$ satisfies statistical parity, i.e. $\hat{y} \bot s $,  $\hat{y}$ is independent with $\hat{s}$ and $\mathcal{L}_R = 0$. 
\label{Theorem:constraint}
\end{theorem}
\begin{proof}
Through $p(\hat{s},\mathbf{h}|s) = p(\hat{s}|s)p(\mathbf{h}|s)$, we could get $p(\mathbf{h}|s,\hat{s})=p(\mathbf{h}|s)$. Then, $p(\hat{y}|s,\hat{s})=p(\hat{y}|s)$ could be derived.  When $\hat{y} \bot s$, the distribution $p(\hat{y},\hat{s})$ would be:
\begin{equation}
\begin{aligned}
p(\hat{y},\hat{s}) & = \sum_s p(\hat{y}|s) p(\hat{s},s) =  \sum_s p(\hat{y}) p(\hat{s},s) = p(\hat{y})p(\hat{s}).
\end{aligned}
\end{equation}
Thus, $\hat{y}$ is independent with $\hat{s}$ when the statistical parity is achieved. Then, we can get $\mathcal{L}_R = |\text{Cov}(\hat{s},\hat{y})| = |\mathbb{E}(\hat{s},\hat{y})-\mathbb{E}(\hat{s})\mathbb{E}(\hat{y})|=0$.
\end{proof}
In the proof, we use the first assumption in Theorem \ref{Theorem:constraint}, which is generally valid as discussed previously.

\subsection{Final Objective Function of FairGNN}
We now have $f_{\mathcal{G}}$ for label prediction, $f_E$ for sensitive attribute estimation, $f_A$ with adversarial debiasing to force the node representations learned by $f_{\mathcal{G}}$ are fair, and covariance constraint to further ensure that the prediction of $f_{\mathcal{G}}$ is fair. Combining all these together, the final objective function could be formulated as:
\begin{equation}
    \min_{\theta_\mathcal{G},\theta_E} \max_{\theta_A} \mathcal{L}_C + \mathcal{L}_E + \alpha  \mathcal{L}_R - \beta \mathcal{L}_A,
    \label{eq:total}
\end{equation}
where $\theta_\mathcal{G}$, $\theta_E$, and $\theta_A$ are the parameters of classifier, estimator, and adversary, respectively. $\alpha$ and $\beta$ are scalars to control the contributions of the covariance constraint and adversarial debiasing. 

\begin{algorithm}[t] 
\caption{ Training Algorithm of FairGNN.} 
\label{alg:Framwork} 
\begin{algorithmic}[1]
\REQUIRE
$\mathcal{G}=(\mathcal{V},\mathcal{E}, \mathbf{X})$, $\mathcal{Y}$, $\mathcal{S}$, $\alpha$ and $\beta$.
\ENSURE $f_{\mathcal{G}}$, $f_A$, and $f_E$
\STATE Initialize $f_E$ by optimizing Eq.(\ref{eq:L_E}) w.r.t $\theta_E$
\REPEAT 
\STATE Obtain the estimated sensitive attributes with $f_E$
\STATE Optimize the GNN classifier parameters $\theta_{\mathcal{G}}$, the adversary parameters $\theta_A$, and the estimator parameters $\theta_E$ by Eq.(\ref{eq:total}). 

\UNTIL convergence
\RETURN $f_{\mathcal{G}}$, $f_A$, and $f_E$
\end{algorithmic}
\end{algorithm}

\subsection{An Training Algorithm of FairGNN}
The training algorithm of FairGNN is presented in Algorithm \ref{alg:Framwork}. Specially, we first pretrain $f_E$ to ensure it meets the second assumption in Theorem \ref{theorem:adv}. Sequentially, we optimize the whole model with Eq.(\ref{eq:total}) through the ADAM optimizer \cite{kingma2014adam}.
In the training process, we replace the hard labels in $\mathcal{L}_{A}$ with soft labels, i.e., the probability produced by $f_E$, to stabilize the adversarial learning \cite{salimans2016improved}.

\begin{table*}[t]
    \small
    \centering
    \caption{The comparisons of our proposed methods with the baselines.}
    \vskip -1.5em
    \begin{tabularx}{\textwidth}{|p{0.06\textwidth}|p{0.06\textwidth}|XX|XXXXXXX|}
    \hline
    Dataset & Metrics & GCN & GAT & ALFR & ALFR-e & Debias & Debias-e & FCGE & FairGCN & FairGAT \\
    \hline
    \hline
    
    \multirow{4}{*}{Pokec-z}
        & ACC (\%) & 70.2 $\pm 0.1$ & 70.4 $\pm 0.1$ & 65.4 $\pm 0.3$ & 68.0 $\pm 0.6$ & 65.2 $\pm 0.7$ & 67.5 $\pm 0.7$ & 65.9 $\pm 0.2$& \textbf{70.0} $\pm \mathbf{0.3}$ &  \textbf{70.1} $\pm \mathbf{0.1}$\\
        & AUC (\%)& 77.2 $\pm 0.1$ & 76.7 $\pm 0.1$ & 71.3 $\pm 0.3$ & 74.0 $\pm 0.7$ & 71.4 $\pm 0.6$ & 74.2 $\pm 0.7$ & 71.0 $\pm 0.2$ & \textbf{76.7} $\pm \mathbf{0.2}$ & \textbf{76.5} $\pm \mathbf{0.2}$\\
        & $\Delta_{SP}$ (\%)& 9.9 $\pm 1.1$ & 9.1 $\pm 0.9$ & 2.8 $\pm 0.5$ & 5.8 $\pm 0.4$ & 1.9 $\pm 0.6$ & 4.7 $\pm 1.0$ & 3.1 $\pm 0.5$ & \textbf{0.9} $\pm \mathbf{0.5}$ & \textbf{0.5} $\pm \mathbf{0.3}$\\
        & $\Delta_{EO}$ (\%)& 9.1 $\pm 0.6$ & 8.4 $\pm 0.6$ & 1.1 $\pm 0.4$ & 2.8 $\pm 0.8$ & 1.9 $\pm 0.4$ & 3.0 $\pm 1.4$ & 1.7 $\pm 0.6$ & \textbf{1.7} $\pm \mathbf{0.2}$ & \textbf{0.8} $\pm\mathbf{0.3}$\\
        \hline
        \hline
    \multirow{4}{*}{Pokec-n}
        & ACC (\%) & 70.5 $\pm 0.2$ & 70.3 $\pm 0.1$ & 63.1 $\pm 0.6$ & 66.2 $\pm 0.5$ & 62.6 $\pm 0.9$ & 65.6 $\pm 0.8$ & 64.8 $\pm 0.5$ & \textbf{70.1} $\pm \mathbf{0.2}$ & \textbf{70.0} $\pm \mathbf{0.2}$\\
        & AUC (\%)& 75.1 $\pm 0.2$ & 75.1 $\pm 0.2$ & 67.7 $\pm 0.5$ & 71.9 $\pm 0.3$ & 67.9 $\pm 0.7$ & 71.7 $\pm 0.7$ & 69.5 $\pm 0.4$ & \textbf{74.9} $\pm \mathbf{0.4}$& \textbf{74.9} $\pm \mathbf{0.4}$\\
        & $\Delta_{SP}$ (\%)& 9.6 $\pm 0.9$ & 9.4 $\pm 0.7$ & 3.05 $\pm 0.5$ & 4.1 $\pm 0.5$ & 2.4 $\pm 0.7$ & 3.6 $\pm 0.2$ & 4.1 $\pm 0.8$ & \textbf{0.8} $\pm \mathbf{0.2}$ & \textbf{0.6} $\pm \mathbf{0.3}$\\
        & $\Delta_{EO}$ (\%)& 12.8 $\pm 1.3$ & 12.0 $\pm 1.5$ & 3.9 $\pm 0.6$ & 4.6 $\pm 1.6 $ & 2.6 $\pm 1.0$ & 4.4 $\pm 1.2$ & 5.5 $\pm 0.9$& \textbf{1.1} $\pm \mathbf{0.5}$ & \textbf{0.8} $\pm \mathbf{0.2}$\\
        \hline
        \hline
    \multirow{4}{*}{NBA}
        & ACC (\%) & 71.2 $\pm 0.5$ & 71.9 $\pm 1.1$ & 64.3 $\pm 1.3$ & 66.0 $\pm 0.4$ & 63.1 $\pm 1.1$  & 65.6 $\pm 2.4$ & 66.0 $\pm 1.5$ & \textbf{71.1} $\pm \mathbf{1.0}$ & \textbf{71.5} $\pm \mathbf{0.8}$\\
        & AUC (\%)& 78.3 $\pm 0.3$ & 78.2 $\pm 0.6$ & 71.5 $\pm 0.3$ & 72.9 $\pm 1.0$ & 71.3 $\pm 0.7$ & 72.9 $\pm 1.2$ & 73.6 $\pm 1.5$ & \textbf{77.0} $\pm \mathbf{0.3}$ & \textbf{77.5} $\pm \mathbf{0.7}$\\
        & $\Delta_{SP}$ (\%)& 7.9 $\pm 1.3$ & 10.2 $\pm 2.5 $ & 2.3 $\pm 0.9$ & 4.7 $\pm 1.8$ & 2.5 $\pm 1.5$ & 5.3 $\pm 0.9$ & 2.9 $\pm 1.0$ & \textbf{1.0} $\pm \mathbf{0.5} $ & \textbf{0.7} $\pm \mathbf{0.5}$\\
        & $\Delta_{EO} (\%)$ & 17.8 $\pm 2.6$ & 15.9 $\pm 4.0$ & 3.2 $\pm 1.5$ & 4.7 $\pm 1.7$ & 3.1 $\pm 1.9$ & 3.1 $\pm 1.3$ & 3.0 $\pm 1.2$ & \textbf{1.2} $\pm \mathbf{0.4}$ & \textbf{0.7} $\pm \mathbf{0.3}$\\
    \hline
    \end{tabularx}
    
    \label{tab:results}
    \vskip -1em
\end{table*}

\section{experiments} \label{sec:experiments}
In this section, we conduct experiments to show the effectiveness of FairGNN for fair node classification. In particular, we aim to answer the following questions:
\begin{itemize}
    \item \textbf{RQ1} Can the proposed FairGNN reduce the bias of GNNs while maintaining high accuracy?

    \item \textbf{RQ2} How do the sensitive attribute estimator, adversarial loss, and covariance constraint affect FairGNN?
    \item \textbf{RQ3} Is FairGNN effective when different amount of sensitive attributes or labels are provided in the training set? 
\end{itemize}
We use the same datasets introduced in Sec.~\ref{sec:datasets} for all the experiments. Next, we will begin by introducing compared methods. 

\subsection{Compared Methods}
We compare our proposed framework with GCN, GAT, and the following representative and state-of-the-art methods for fair classification and fair graph embedding learning:
\begin{itemize}[leftmargin=*]

    \item \textbf{ALFR} \cite{edwards2015censoring}: This is a pre-processing method. A discriminator is applied to remove the sensitive information in the representations produced by a MLP-based autoencoder. Then, linear classifier is trained on the debiased representations.
    \item \textbf{ALFR-e}: To utilize the graph structure information, ALFR-e concatenates
    the graph embeddings learned by deepwalk \cite{perozzi2014deepwalk} with the user features in the ALFR.
    \item \textbf{Debias} \cite{zhang2018mitigating}: This is an in-processing fair classification method.  It directly applies an discriminator on the estimated probability of classifier $\eta: \mathbf{x} \rightarrow \mathbb{R}$. It would make the probability distribution $p(\eta(\mathbf{x})|s=0)$ closer to $p(\eta(\mathbf{x})|s=1)$.
    \item \textbf{Debias-e}: Similar to the ALFR-e, we also add the deepwalk embeddings to the features used in Debias.
    \item \textbf{FCGE} \cite{bose2019compositional}: FCGE is proposed to learn fair node embeddings  in graph without node features through edge prediction. The sensitive information in the embeddings is filtered by discriminators.
\end{itemize}
ALFR and ALFR-e are trained with features of all the users $\mathcal{V}$, labels of $\mathcal{V}_L$, and the sensitive attributes of $\mathcal{V}_S$ for fair classification.
Debis and Debias-e require the sensitive attributes of labeled nodes, which is on contrary with our setting that $\mathcal{V}_L$ could have no overlap with $\mathcal{V}_S$. Thus, we use the estimated labels of $\mathcal{V}_S$, features of $\mathcal{V}_L$, and labels of $\mathcal{V}_L$ to train Debias and Debias-e. FCGE utilizes $\mathcal{G}$, labels of $\mathcal{V}_L$, and sensitive attributes of $\mathcal{V}_S$.

For FairGNN, we deploy a one hidden layer GCN for $f_E$. The hidden dimension is set as 128. We use a linear classifier for $f_A$. To verify that our framework is useful for various GNNs, we adopt both GCN and GAT as the backbone of the FairGNN classifier $f_{\mathcal{G}}$, which are named as \textbf{FairGCN} and \textbf{FairGAT}. In FairGCN, the GCN classifier contains one hidden layer with dimension 128. The GAT 
classifier in FairGAT also contains two layers in total. We set the number of heads as 1. The dimensions of the GAT classifiers' hidden layer for Pokec-z, Pokec-n and NBA are 64, 64 and 32, respectively.

\subsection{Fair Classification on Graph}

To answer \textbf{RQ1}, we evaluate our proposed FairGNN in terms of fairness and classification performance. 
$\Delta_{SP}$ and $\Delta_{EO}$ are used to show the discrimination level, which are introduced in Section~\ref{sec:fairness}. The smaller $\Delta_{SP}$ and $\Delta_{EO}$ are, the more fair the classifier is. Accuracy (ACC) and ROC AUC score are used to evaluate the classification performance. For all the models, we tune the hyperparameters on the training set via cross validation. 
For FairGCN, we set $\alpha$ to 100 and $\beta$ to 1.
For FairGAT, $\alpha$ is 2 and $\beta$ is 0.1. More details about hyperparameter selection will be discussed in Sec~\ref{sec:parameter_sensitivity}. All the experiments are conducted 5 times. 
The mean and standard deviations for all the models on the three datasets are reported in Table \ref{tab:results}. From the table, we make the following observations:
\begin{itemize}[leftmargin=*]
    \item Compared with GCN and GAT, the general fair classification methods and graph embeddings learning method show poor performance in classification even with the help of graph information, while FairGCN and FairGAT perform very close to the based GNNs. This suggests the necessity of investigating fair classification algorithms on GNNs for accurate predictions;
    \item Under the condition of limited sensitive information, baselines show obvious bias and the ones utilizing graph information are even worse. On the contrary, our proposed models obtain $\Delta_{SP}$ and $\Delta_{EO}$ that are close to 0, which indicates that the discrimination is basically eliminated; and
    \item FairGAT is slightly better than FairGCN in Fairness. This is reasonable because the learnable edge coefficients in GAT could be helpful to reduce the weights of the edges that bring bias. 
\end{itemize}
These observations demonstrate the effectiveness of our proposed framework in making fair and accurate predictions.

\subsection{Ablation Study}
To answer \textbf{RQ2}, we conduct ablation studies to understand the impacts of $f_E$, adversarial loss, and covariance constraint.

\subsubsection{Impact of $f_E$}  In our proposed framework, a GCN estimator is deployed to predict  sensitive attributes for adversarial debiasing. 
To show the importance of the GCN estimator, we analyze it from two aspects.
Firstly, to demonstrate the effectiveness of the noisy sensitive attributes, we eliminate the estimator and only use the provided sensitive attributes $\mathcal{S}$ to get a variant denoted as FairGNN$\backslash$E. Secondly, to investigate how a weaker estimator would influence the fair classification, we train a variant $\text{FairGNN}_{MLP}$ by using MLP as the estimator. Hyperparameters of 
these variants are determined by cross validation with gird search. Specifically, we vary $\alpha$ and $\beta$ among $\{0.0001,0.001,0.1,1\}$ and $\{1, 2, 5, 10, 20, 50, 100\}$, respectively. 
For each variant, the experiments are conducted 5 times. The average performance of fairness in terms of $\Delta_{SP}$ and node classificaiton in terms of AUC on Pockec-z are presented in Fig. \ref{fig:abl}(a) and (b), respectively. We only show the results on Pockec-z as we have similar observations on the other datasets. From the figures, we make the following observations: 
\begin{itemize}[leftmargin=*]
    \item The $\Delta_{SP}$ score of FairGNN$\backslash$E is much larger than that of FairGNN. which is because the provided sensitive attributes are inadequate. This shows that $f_E$ plays an important role in FairGNN; and 
    \item The performance of sensitive attribute prediction in terms of AUC for MLP estimator is 0.69, which is much lower than that of GCN estimator, which is 0.80. Though $\text{FairGNN}_{MLP}$ adopts a much weaker estimator than FairGNN, the performance in terms of fairness is slightly worse than FairGNN. This aligns with our theoretical analysis that $f_E$ doesn't need to be very accurate. However, the marginal differences still indicate that too much noise in sensitive attributes may still slightly affect the fairness.
\end{itemize}

\begin{figure}[t]
\centering
\begin{subfigure}{0.49\columnwidth}
    \centering
    \includegraphics[width=\linewidth]{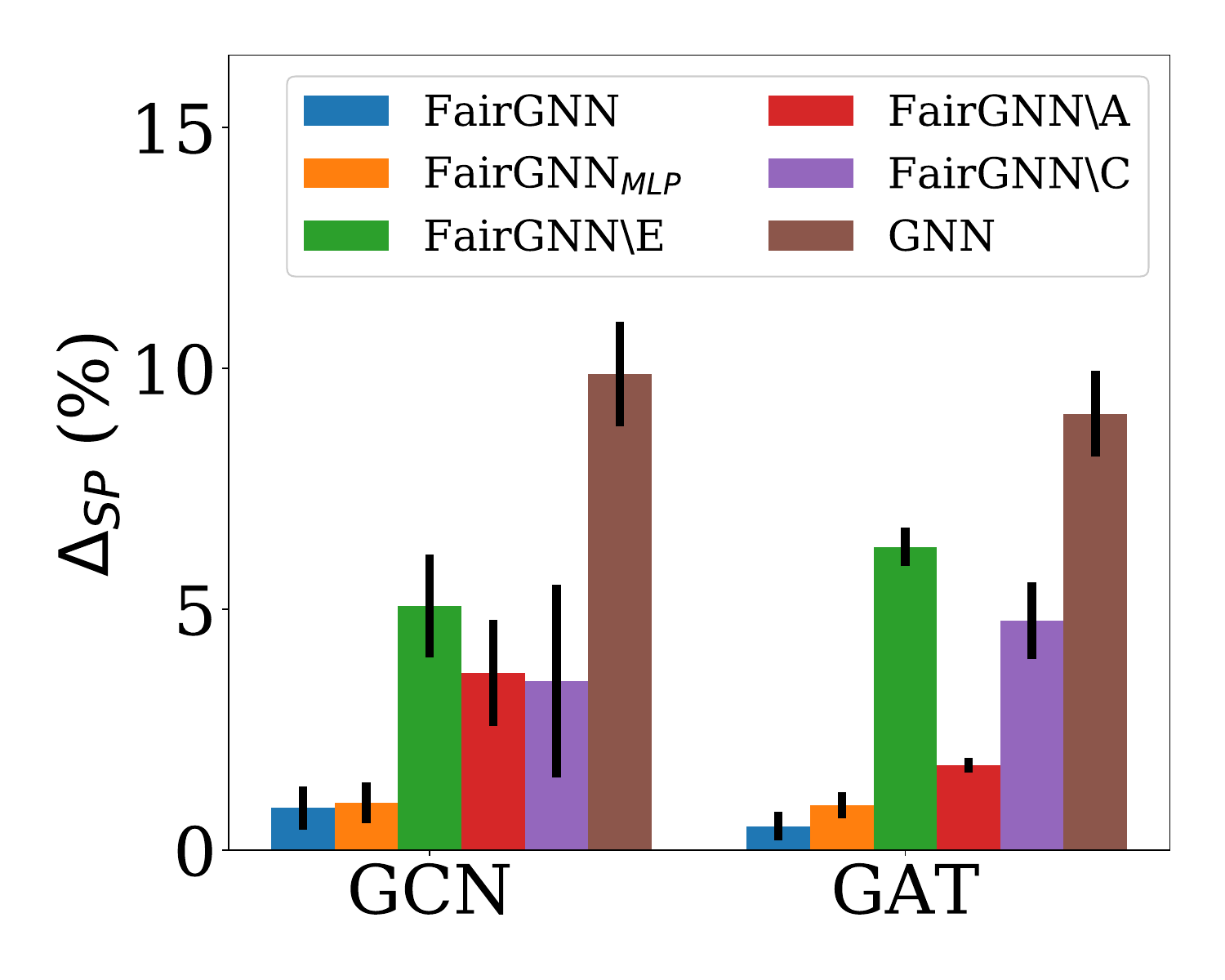} 
    \vskip -0.5em
    \caption{$\Delta_{SP}$}
    \label{fig:abl_sp}
\end{subfigure}
\begin{subfigure}{0.49\columnwidth}
    \centering
    \includegraphics[width=\linewidth]{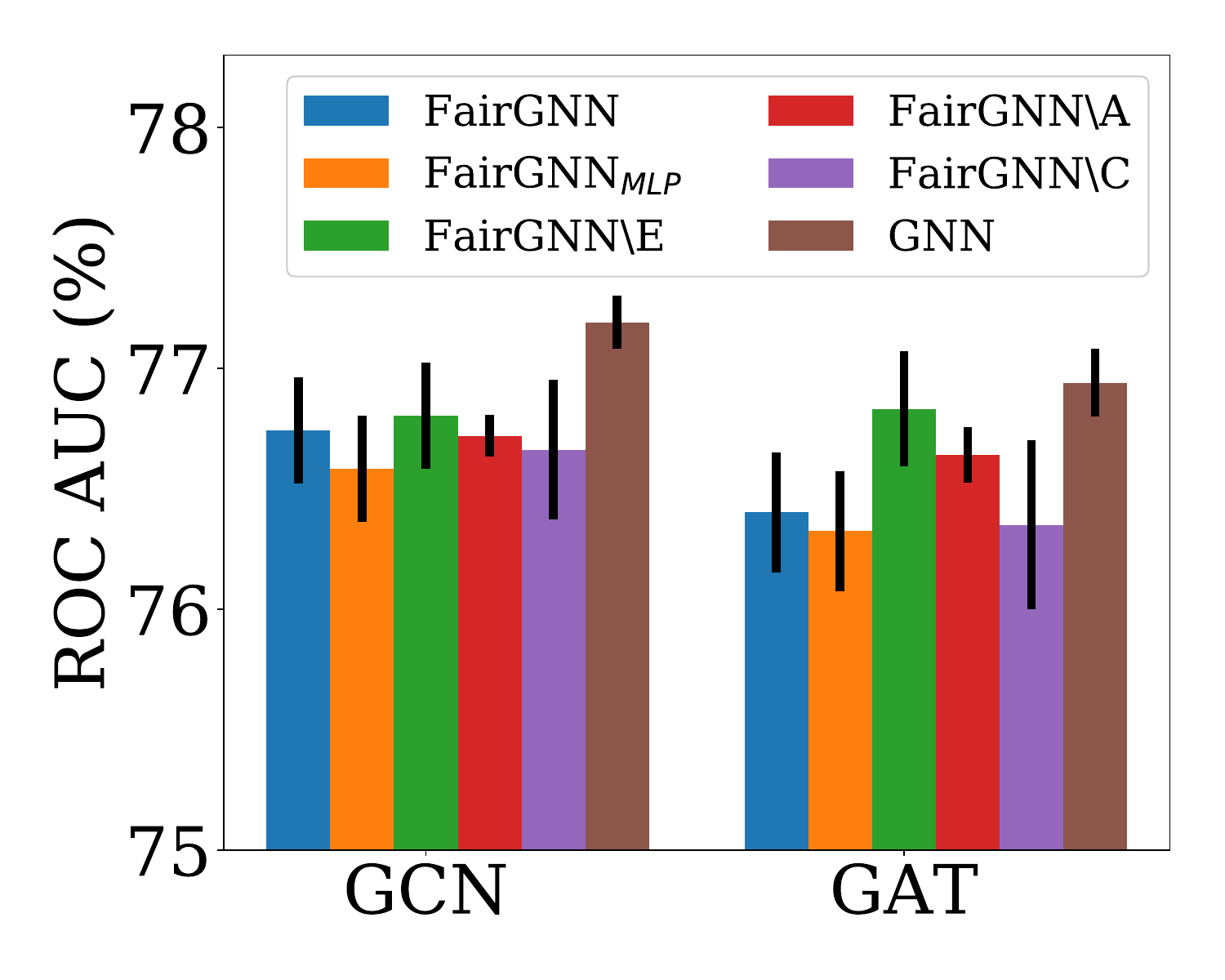} 
    \vskip -0.5em
    \caption{ROC AUC}
    \label{fig:abl_roc}
\end{subfigure}
 \vspace{-1.5em}
\caption{Comparisons between FairGNN and its variants.}
\vspace{-1em}
\label{fig:abl}
\end{figure}

\subsubsection{Impacts of the adversarial debiasing and covariance constraint} 
To demonstrate the effects of the adversarial loss and covariance constraint, we train two variants of FairGNN, i.e., FairGNN$\backslash$A and FairGNN$\backslash$C, where FairGNN$\backslash$A means FairGNN without the adversarial loss, and FairGNN$\backslash$C means FiarGNN without covariance constraint. Similarly, for each variant, we run the experiment 5 times on Pokec-z and the average performances are shown in Figure \ref{fig:abl}. From the figure, we observe: 
\begin{itemize}[leftmargin=*]
    \item The $\Delta_{SP}$ scores for both FairGNN$\backslash$C and FairGNN$\backslash$A are much smaller than that of GNNs in Figure \ref{fig:abl}, which shows that both covariance constraint and adversarial debiasing can improve fairness; and
    \item The $\Delta_{SP}$ scores for both FairGNN$\backslash$C and FairGNN$\backslash$A are much larger than that of FairGNN, which implies that using both covariance constraint and adversarial debiasing can achieve better fairness. This is because they regularize the GNN from two different perspectives, i.e., adversarial debiasing regularizes on the node representations while covariance cosntraint is directly on the predictions for fair classification.
\end{itemize}

\subsection{Impacts of Sizes of $\mathcal{V}_S$ and $\mathcal{V}_L$}
To answer \textbf{RQ3}, we study the impacts of the sizes of $\mathcal{V}_S$ and  $\mathcal{V}_L$ on FairGAT. 
 We set $\alpha=0.1$ and $\beta=2$ based on cross validation.
We vary $|\mathcal{V}_S|$ as $\{200,600,1000,1400,1800,2200,2600,3000\}$. Each experiment is conducted 5 times and the average results on Pokec-z with comparison to FairGAT$\backslash$E and ALFR-e are  shown in Fig.~\ref{fig:sens_compare}. 
From the figure, we observe that: 
(i) Generally, both FairGAT$\backslash$E and ALFR-e have high discrimination scores when  $|\mathcal{V}_S|$ is small. They need plenty of data with sensitive attributes to become effective. FairGAT could get very low $\Delta_{SP}$ even when  $|\mathcal{V}_S|$ is as small as 200. This implies that FairGAT is insensitive to the size of data with sensitive attributes, which is because we have $f_E$ to estimate the sensitive attributes. Though extremely small $|\mathcal{V}_S|$ would lead to a weak $f_E$, we still have similar $\Delta_{SP}$ score as that when $\mathcal{V}_S$ is large. This verifies  our theoretical analysis that we can achieve good fairness with a reasonable $f_E$; 
(ii) FairGAT$\backslash$E and ALFR-e decrease slightly 
in classification performance with the increasing of the size of $\mathcal{V}_S$, which is because more data with sensitive attribute would lead to a stricter regularization. In the contrary, FairGAT keeps high classification performance and even perform slightly better with more sensitive attributes. This is because the size of sensitive attributes $\hat{\mathcal{S}}$ used for training FairGAT are fixed to the size of $\mathcal{V}$, and less noise in  the estimation of the sensitive attributes is helpful to better learn representations for classification.

Similarly, we vary $|\mathcal{V}_L|$ as $\{500, 1000, 1500, 2000\}$ and each experiment is run for 5 times. The average results on Pokec-z are reported in Figure \ref{fig:sens_sp_l}. 
We only report the results on Pokec-z as we have simialr observations on other datasets. 
From the figure, we observe that: FairGAT consistently shows effectiveness in eliminating discrimination. The drop in classification performance is marginal. 
This demonstrates that our proposed method could achieve fairness while keep high accuracy in general  scenarios which correspond to various sizes of $\mathcal{V}_S$ and $\mathcal{V}_L$.

\begin{figure}[t]
\centering
\begin{subfigure}{0.45\columnwidth}
    \centering
    \includegraphics[width=\linewidth]{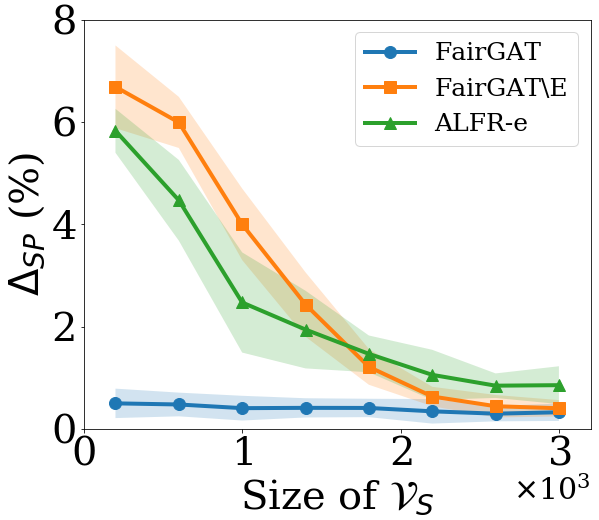} 
    \vspace{-5mm}
    \caption{$\Delta_{SP}$}
    \label{fig:sens_sp}
\end{subfigure}
\begin{subfigure}{0.45\columnwidth}
    \centering
    \includegraphics[width=\linewidth]{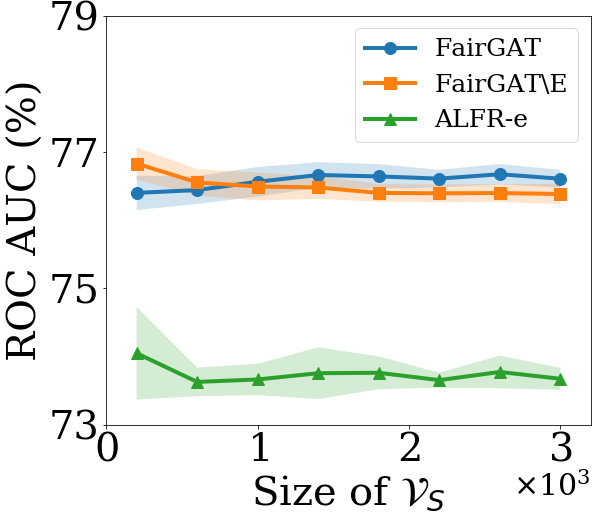} 
    \vspace{-5mm}
    \caption{ROC AUC}
    \label{fig:sens_roc}
\end{subfigure}
\vskip -1.2em
\caption{Impacts of the size of $\mathcal{V}_S$ to FairGAT.}
\vskip -1em
\vspace{-0.5em}
\label{fig:sens_compare}
\end{figure}

\begin{figure}[t]
\centering
\begin{subfigure}{0.45\columnwidth}
    \centering
    \includegraphics[width=\linewidth]{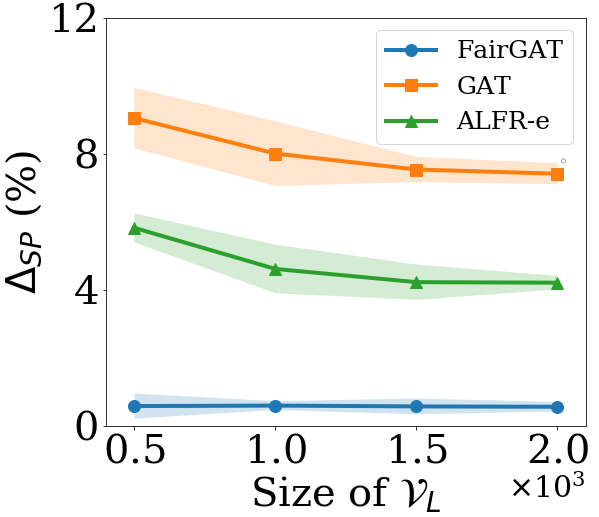} 
    \vspace{-5mm}
    \caption{$\Delta_{SP}$}
    \label{fig:sens_sp_l}
\end{subfigure}
\begin{subfigure}{0.45\columnwidth}
    \centering
    \includegraphics[width=\linewidth]{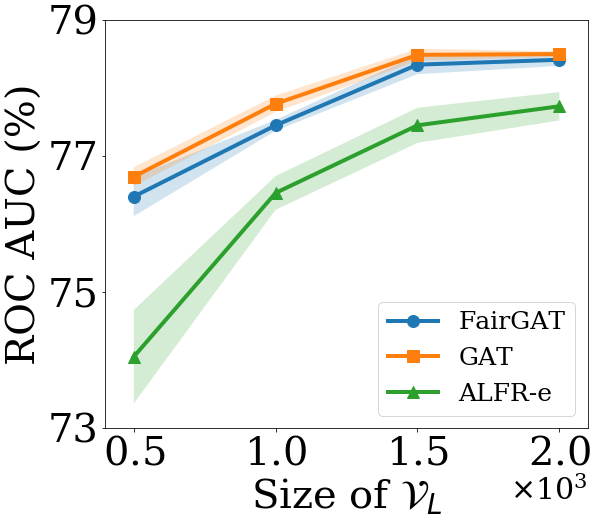} 
    \vspace{-5mm}
    \caption{ROC AUC}
    \label{fig:sens_roc_l}
\end{subfigure}
\vskip -1.2em
\caption{Impacts of the size of $\mathcal{V}_L$ to FairGAT.}
\vspace{-1em}
\label{fig:sens_compare_l}
\end{figure}

\subsection{Parameter Sensitivity} \label{sec:parameter_sensitivity}
There are two important hyperparameters in our proposed model, .i.e., $\alpha$ controlling the influence of the adversary to the GNN classifier, while $\beta$ controlling the contribution of the covariance constraint to ensure fairness. 
To investigate the parameter sensitivity and find the ranges that achieve high accuracy with low discrimination score, we train FairGAT models on Pokec-z with various hyperparameters.
More specifically, we alter the values of $\alpha$ and $\beta$ among $\{0.0001,0.001,0.01,0.1,1\}$ and $\{1,2,5,10,20,50,100\}$. 
The results are presented in Figure \ref{fig:para}.
From Figure \ref{fig:para} (\subref{fig:abl_roc}), we can find that when $\alpha \leq 0.01$ and $\beta \leq 20$ the classification performance is almost unaffected. Once $\alpha$ and $\beta$ are too large, the classifier's performance will decay rapidly. The impacts of the hyperparameters to the discrimination score are presented in Figure \ref{fig:para} (\subref{fig:abl_sp}). When we increase the value of $\alpha$, $\Delta_{SP}$ will 
firstly decrease as expected. Then, it would increase when the value of $\alpha$ is too large. Because it would be difficult to optimize the GNN classifier to the global minimum when the contribution of the adversary is extremely high. As for $\beta$, the discrimination score would consistently reduce when we increase its value. Combining the two figures, we could determine that when $\alpha \in [0.001, 0.01]$ and $\beta \in [5,20]$, the GNN classifier achieves fairness and maintains high node classification accuracy.
\begin{figure}[t]
\centering
\begin{subfigure}{0.45\columnwidth}
    \centering
    \includegraphics[width=\linewidth]{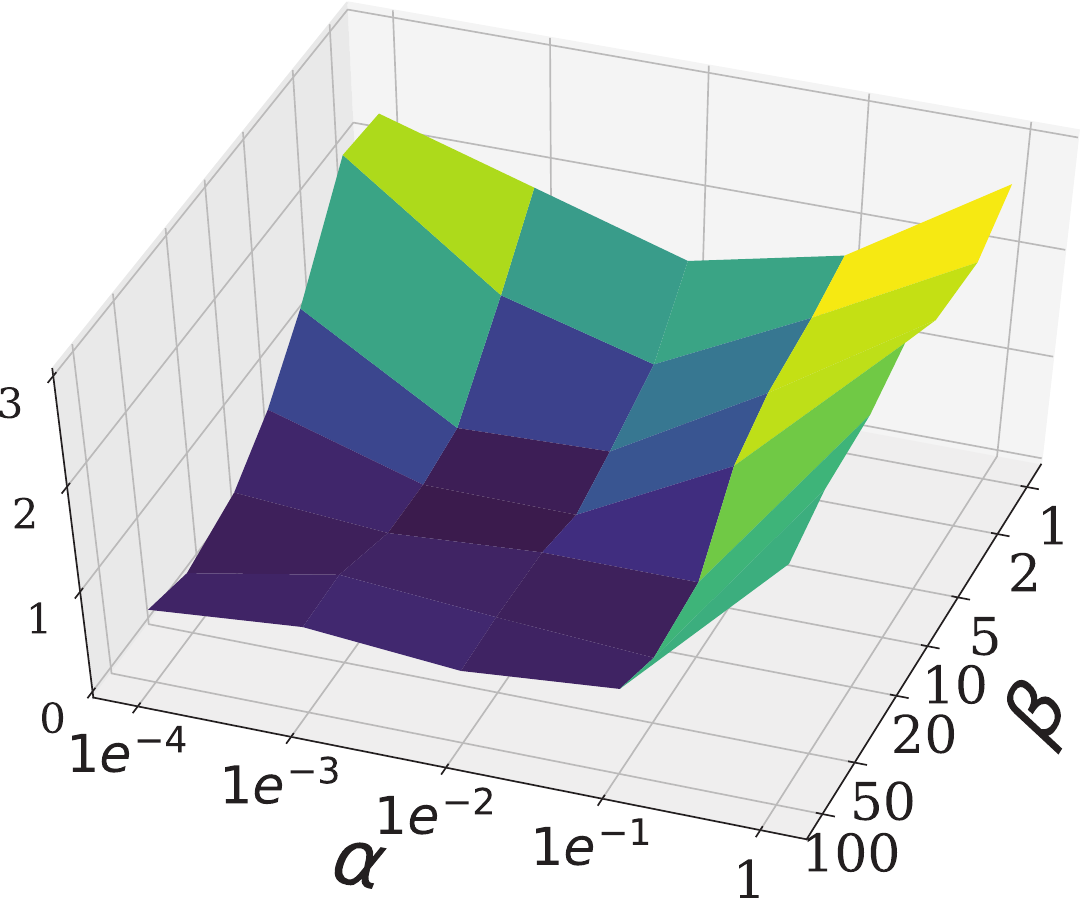} 
    \vspace{-5mm}
    \caption{$\Delta_{SP}$ (\%)}
    \label{fig:abl_sp}
\end{subfigure}
\begin{subfigure}{0.45\columnwidth}
    \centering
    \includegraphics[width=\linewidth]{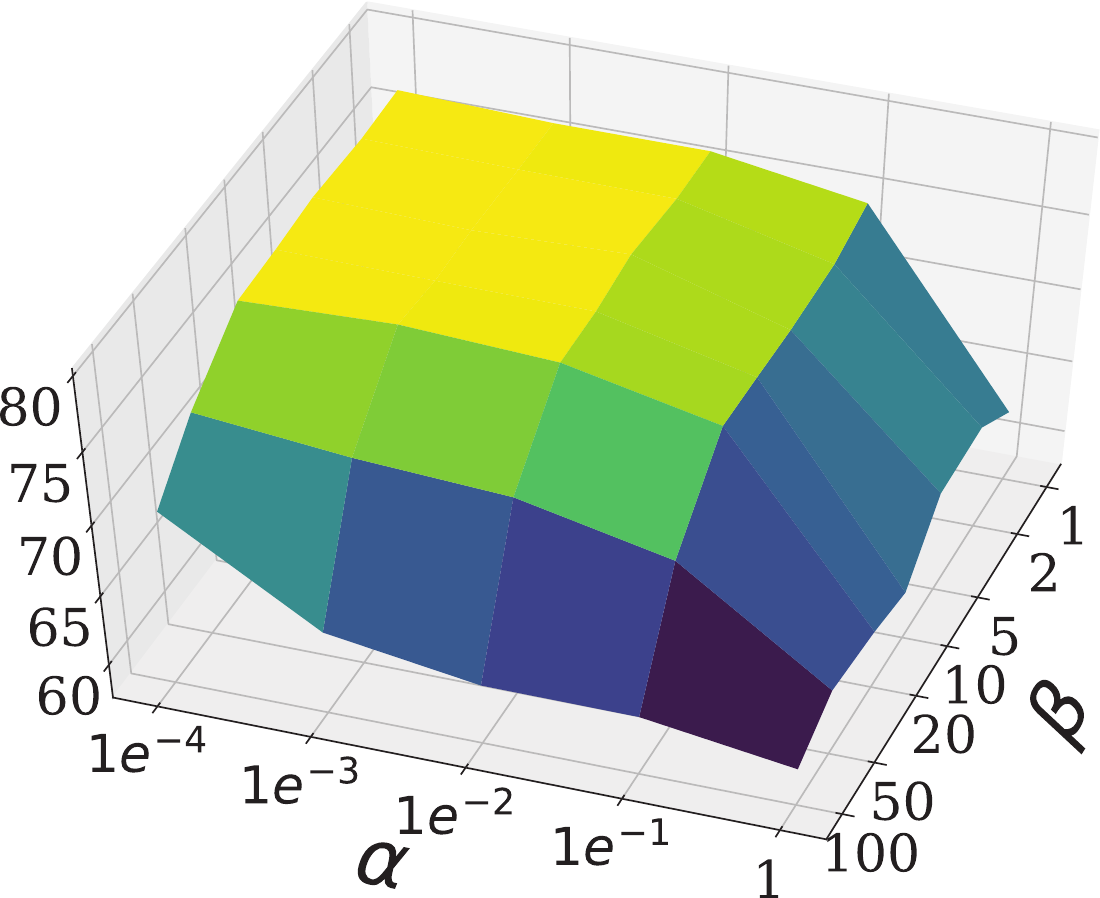} 
    \vspace{-5mm}
    \caption{ROC AUC (\%)}
    \label{fig:abl_roc}
\end{subfigure} ~~~
\vskip -1.3em
\caption{Parameter sensitivity analysis.}
\vskip -1.5em
\label{fig:para}
\end{figure}

\section{conclusion and future work} \label{sec:conclusion}
In this paper, we study a novel problem of fair GNN learning with limited sensitive information. We empirically demonstrate that GNNs exhibit severe bias. We propose a novel and flexible framework FairGNN which is able to significantly alleviate the bias issue of GNNs meanwhile  maintain high performance on node classification. FairGNN adopts a sensitive attribute estimator to alleviate the issue of lacking sensitive attribute information. With the estimated sensitive attributes, FairGNN designs adversarial debiasing and covariance constraint to regularize the GNN to have fair node representations and predictions, respectively. We theoretically show that FairGNN can reduce the bias. 
Experiment results on real-world datasets demonstrate the effectiveness of the proposed framework in terms of both fairness and classification performance. 
There are several interesting directions which need further investigation. First, we assume the provided  sensitive attributes are clean. However, for some applications in social media, users might randomly input sensitive attributes such as gender due to privacy concern. Thus, we will extend FairGNN to deal with limited and inaccurate sensitive information. Second, the experiments show that the edges are possible to bring bias. Thus, we will also explore methods which add/delete links in graphs to improve the fairness and classification performance of FairGNN.
\section{Acknowledgements}
This material is based upon work supported by, or in part by, the National Science Foundation (NSF) under grant IIS-1909702, IIS-1955851, and the Global Research Outreach program of Samsung Advanced Institute of Technology under grant \#225003. The findings and conclusions in this paper do not necessarily reflect the view of the funding agency.


\bibliographystyle{ACM-Reference-Format}
\bibliography{ref}

\end{document}